\documentclass[anon]{alt2026arxiv} 

\usepackage{amssymb}
\usepackage{amsthm}

\newtheorem{theorem}{Theorem}[section]
\newtheorem{lemma}[theorem]{Lemma}
\newtheorem{corollary}[theorem]{Corollary}
\newtheorem{proposition}[theorem]{Proposition}

\newcommand{\I}{\mathbf{I}}

\usepackage[T1]{fontenc}
\usepackage[utf8]{inputenc}
\usepackage{times}
\usepackage{microtype}

\usepackage{mathtools}   
\usepackage{bm}          
\usepackage{bbm}         
\usepackage{booktabs}    

\usepackage{enumitem}    

\usepackage{capt-of}  
\usepackage{graphicx}
\usepackage{subcaption}  
\DeclareGraphicsExtensions{.pdf,.png,.jpg}
\usepackage{threeparttable}
\usepackage{algorithm}
\usepackage{algpseudocode}

\usepackage{tikz}
\usetikzlibrary{positioning}

\usepackage{thmtools}

\usepackage{hyperref}
\hypersetup{
  colorlinks=true,
  linkcolor=blue,
  citecolor=blue,
  urlcolor=blue
}

\usepackage{cleveref}
\newcommand{\ball}[2]{B_{#1}\!\left(#2\right)}
\newcommand{\opnorm}[1]{\left\lVert #1 \right\rVert_{\mathrm{op}}}
\newcommand{\E}{\mathbb{E}}
\newcommand{\Pbb}{\mathbb{P}}
\newcommand{\R}{\mathbb{R}}
\usepackage{natbib}    
\crefname{assumption}{Assumption}{Assumptions}
\Crefname{assumption}{Assumption}{Assumptions}
\crefname{theorem}{Theorem}{Theorems}
\Crefname{theorem}{Theorem}{Theorems}
\crefname{lemma}{Lemma}{Lemmas}
\Crefname{lemma}{Lemma}{Lemmas}
\crefname{proposition}{Proposition}{Propositions}
\Crefname{proposition}{Proposition}{Propositions}
\crefname{corollary}{Corollary}{Corollaries}
\Crefname{corollary}{Corollary}{Corollaries}
\crefname{algorithm}{Algorithm}{Algorithms}
\Crefname{algorithm}{Algorithm}{Algorithms}
\crefname{equation}{Eq.}{Eqs.}
\Crefname{equation}{Equation}{Equations}

\newcommand{\norm}[1]{\left\lVert #1 \right\rVert}

\newcommand{\ip}[2]{\left\langle #1,\,#2 \right\rangle}

\usepackage{times}

\title{Spectral Thresholds for Identifiability and Stability:\\
Finite-Sample Phase Transitions in High-Dimensional Learning}

\author{William Hao-Cheng Huang\\
Taiwan Semiconductor Manufacturing Company (TSMC)\\
\texttt{williamhuang0709@gmail.com}}

\date{\small Preprint. A conference version is under review.}

\begin{document}

\maketitle

\begin{abstract}
In high-dimensional learning, models remain stable until they collapse abruptly once the sample size falls below a critical level. This instability is not algorithm-specific but a geometric mechanism: when the weakest Fisher eigendirection falls beneath sample-level fluctuations, identifiability fails. Our Fisher Threshold Theorem formalizes this by proving that stability requires the minimal Fisher eigenvalue to exceed an explicit $O(\sqrt{d/n})$ bound. Unlike prior asymptotic or model-specific criteria, this threshold is finite-sample and necessary, marking a sharp phase transition between reliable concentration and inevitable failure. To make the principle constructive, we introduce the Fisher floor, a verifiable spectral regularization robust to smoothing and preconditioning. Synthetic experiments on Gaussian mixtures and logistic models confirm the predicted transition, consistent with $d/n$ scaling. Statistically, the threshold sharpens classical eigenvalue conditions into a non-asymptotic law; learning-theoretically, it defines a spectral sample-complexity frontier, bridging theory with diagnostics for robust high-dimensional inference.
\end{abstract}


\section{Introduction}

In modern high-dimensional learning, models often appear reliable up to a point, only to collapse abruptly once the sample size falls below a critical level. This is evident in double descent in overparameterized neural networks \citep{belkin2019reconciling} and in high-dimensional regression, where estimation becomes unreliable once $n$ is on the order of $d$. Existing frameworks---such as Fisher consistency, restricted eigenvalue conditions, or information-theoretic bounds---offer only asymptotic guarantees or loose sufficient criteria, leaving these sharp transitions unexplained.

\medskip
\noindent\textbf{Our goal.} We seek a finite-sample law that separates stability from inevitable failure, independent of algorithmic specifics. Such a law should act as a sharp identifiability criterion, clarifying when inference is possible and when estimation must collapse.

\medskip
\noindent\textbf{Main result.} Our central theorem establishes a Fisher spectral threshold: identifiability requires the minimal eigenvalue of the empirical Fisher information to exceed an explicit $O(\sqrt{d/n})$ bound. Unlike asymptotic Fisher consistency \citep{lecam1970asymptotic}, this threshold is both necessary and finite-sample: above it, parameters concentrate reliably; below it, estimation fails due to Fisher spectrum degeneracy, as weak eigendirections become indistinguishable under finite-sample noise. This refines prior phase-transition analyses in spiked models \citep{baik2005phase} and double descent \citep{belkin2019reconciling}, yielding a sharp non-asymptotic boundary for when learning remains possible. Detailed proofs, including the PL inequality, are in Section~4, building on assumptions in Section~3.

\medskip
\noindent\textbf{Key advances.} Our approach advances the field by: first, proving a \emph{Fisher Threshold Theorem} that establishes a necessary finite-sample spectral law for identifiability; second, introducing a \emph{Constructive Fisher Floor}, a verifiable regularization robust to smoothing and preconditioning; and third, verifying the threshold in synthetic experiments on Gaussian mixtures and logistic models, consistent with the predicted $d/n$ scaling and visualized in Figure~1.

\medskip
\noindent\textbf{Implications.} Statistically, the Fisher threshold sharpens regression eigenvalue conditions into a finite-sample law; learning-theoretically, it defines a spectral sample-complexity frontier.


\section{Related Work}

\paragraph{Statistical Identifiability.}
Classical asymptotic statistics links identifiability to Fisher information,
via local asymptotic normality and the Cramér--Rao inequality
\citep{lecam1970asymptotic,vandervaart1998asymptotic}.
In high-dimensional settings, conditions such as restricted eigenvalue
and restricted strong convexity \citep{bickel2009simultaneous,negahban2012unified}
yield sufficient guarantees, while modern analyses reveal sharp feasibility
boundaries for specific MLEs \citep{sur2019modern}.
However, these results are either asymptotic, provide only loose sufficient criteria,
or remain tied to particular model classes, leaving open whether there exists
a verifiable \emph{necessary} law that dictates when identifiability must collapse.
Our contribution addresses this gap by providing a \emph{general necessary finite-sample spectral law},
reframing Fisher information as a concrete non-asymptotic criterion for identifiability.

\paragraph{Spectral Phase Transitions.}
Phase-transition phenomena are central in high-dimensional inference:
the BBP transition in spiked models \citep{baik2005phase},
sparse PCA \citep{lesieur2015phase}, and multi-index models \citep{defilippis2025optimal}.
In machine learning, related instabilities appear through the ``double descent'' phenomenon \citep{belkin2019reconciling}
and analyses of Fisher information spectra in deep networks
\citep{pennington2018spectrum,karakida2019universal}.
These works demonstrate that spectral degeneracies often coincide with instability,
but their conclusions remain either tied to specific models or descriptive in nature,
and therefore stop short of providing general, verifiable necessary thresholds.
Our results sharpen these insights by establishing an explicit spectral boundary
that marks the onset of stability failure, connecting Fisher spectrum degeneracy
directly to finite-sample identifiability as a necessary law.

\paragraph{Algorithmic Stability and Generalization.}
Within learning theory, stability and generalization have been studied through
uniform stability \citep{bousquet2002stability,hardt2016train},
PAC-Bayesian analysis \citep{mcallester1999some,dziugaite2017computing},
and information-theoretic approaches \citep{xu2017information}.
These frameworks largely provide \emph{sufficient} guarantees, ensuring generalization
when stability holds, but do not characterize when stability must necessarily fail.
Lower bounds, such as those of \citet{feldman2018generalization}, highlight the inherent
limits of uniform stability, but remain tied to specific algorithmic assumptions.
Our contribution complements this line by establishing a \emph{spectral lower bound on stability}:
an algorithm-independent criterion that becomes binding once Fisher curvature
falls below sample-level fluctuations, thereby marking a fundamental and verifiable
impossibility frontier for learning algorithms.

\paragraph{Positioning.}
In summary, prior work has clarified asymptotic identifiability, demonstrated
empirical spectral instabilities, and established sufficient stability criteria.
Yet these strands have remained fragmented: asymptotic laws ignore finite-sample
instabilities, empirical spectra lack necessity, and stability bounds emphasize
sufficiency. Our Fisher threshold unifies and refines these directions by
redefining Fisher information as a finite-sample phase-transition law and
providing an algorithm-independent lower bound on stability.
This bridge between statistical identifiability and learning-theoretic stability
sets the stage for our formal development in the next section.


\section{Preliminaries}

We begin by introducing the structural assumptions that form the analytical backbone of our results. 
Rather than treating them as merely technical conditions, we emphasize their role as a 
\emph{bridge} between optimization geometry and statistical identifiability: smoothness 
translates optimization arguments into quantitative inequalities, concentration lifts 
population curvature to the sample level, and KL control quantifies the statistical indistinguishability of local alternatives. 
Taken together, these assumptions—and their immediate consequences—will reappear verbatim 
across theorems and experiments, serving as the common ``calculus rules'' of our analysis.

\paragraph{Setup and Notation.}
We observe $n$ i.i.d.\ samples $(X_i,Y_i)_{i=1}^n$ from a parametric model 
$\{P_\theta: \theta \in \Theta \subset \mathbb{R}^d\}$. 
The per-sample loss is denoted $\ell(\theta;X,Y)$ and the empirical risk is 
$L(\theta) = \tfrac{1}{n} \sum_{i=1}^n \ell(\theta;X_i,Y_i)$. 
Fix a reference parameter $\theta^\ast$ (typically the population minimizer). 
The population Fisher information at $\theta^\ast$ is
\[
\Gamma = \mathbb{E}[\,s(\theta^\ast)s(\theta^\ast)^\top\,], 
\qquad s(\theta) := \nabla_\theta \ell(\theta;X,Y).
\]
We write its eigenvalues in decreasing order $\lambda_1 \ge \cdots \ge \lambda_d =: \lambda_{\min}$. 
The operator norm is denoted $\|\cdot\|_{\mathrm{op}}$. 
For $r>0$, we let $B_r(\theta^\ast) = \{\theta : \|\theta - \theta^\ast\|\le r\}$.

\paragraph{Assumptions.}
Throughout we impose the following local conditions around $\theta^\ast$:

\begin{itemize}[leftmargin=*]
\item[(A1)] \textbf{Local smoothness.}  
There exist $r>0$ and $L_{\mathrm{sm}}>0$ such that $\nabla L$ is $L_{\mathrm{sm}}$-Lipschitz on $B_r(\theta^\ast)$.  
\emph{Interpretation:} the loss surface has no abrupt curvature spikes, ensuring Taylor expansions 
and descent arguments apply uniformly.  

\item[(A2)] \textbf{Robust concentration of the empirical Fisher.}  
There exist $\sigma_{\mathrm{eff}}>0$ and $C_0>0$ such that with probability at least $1-\delta$,  
\[
\|\widehat{\Gamma} - \Gamma\|_{\mathrm{op}} \le C_0\sigma_{\mathrm{eff}} \sqrt{\tfrac{d+\log(1/\delta)}{n}}
\;=:\;\Lambda^\ast,
\qquad \widehat{\Gamma} := \tfrac{1}{n}\sum_{i=1}^n s_i s_i^\top, \quad s_i := \nabla_\theta \ell(\theta^\ast;X_i,Y_i).
\]  
\emph{Interpretation:} this assumption ensures that empirical curvature tracks the population curvature up to sampling fluctuations, preventing informative directions from vanishing in finite samples.

\item[(A3)] \textbf{Local quadratic KL upper bound (LAN-type control).}  
There exists $C_{\mathrm{KL}}>0$ and $r>0$ such that for all $\theta\in B_r(\theta^\ast)$,  
\[
\mathrm{KL}(P_\theta \,\|\, P_{\theta^\ast}) \le \tfrac{C_{\mathrm{KL}}}{2} (\theta-\theta^\ast)^\top \Gamma (\theta-\theta^\ast).
\]  
\emph{Interpretation:} the model admits a local asymptotic normality (LAN) expansion at $\theta^\ast$, 
so statistical distinguishability grows quadratically in the Fisher metric. 
\end{itemize}

\paragraph{Frequently Used Consequences.}
From (A1)–(A3), we will repeatedly invoke three consequences:

\begin{itemize}[leftmargin=*]
\item[(C1)] \emph{Descent Lemma.}  
$L(\theta) - L(\theta^\ast) \le \tfrac{L_{\mathrm{sm}}}{2}\|\theta-\theta^\ast\|^2$.  
\item[(C2)] \emph{Spectral perturbation (Weyl).} \citep{stewart1990matrix} 
$\lambda_{\min}(\widehat{\Gamma}) \ge \lambda_{\min}(\Gamma) - \Lambda^\ast$, 
where $\Lambda^\ast = C_0\sigma_{\mathrm{eff}}\sqrt{\tfrac{d+\log(1/\delta)}{n}}$.  
\item[(C3)] \emph{Local two-point KL bound.}  
For any unit vector $v$ and $\rho>0$ with $\theta^\ast \pm \rho v \in B_r(\theta^\ast)$,
\[
\mathrm{KL}(P_{\theta^\ast+\rho v}\,\|\, P_{\theta^\ast-\rho v}) \le C'_{\mathrm{KL}} \rho^2 \lambda_{\min},
\]
for some $C'_{\mathrm{KL}} \in [C_{\mathrm{KL}},2C_{\mathrm{KL}}]$ depending only on local Fisher comparability.  
\end{itemize}

\paragraph{Role in the Paper.}
Geometrically, (C1) converts distances into function-value gaps, (C2) lifts Fisher concentration into 
finite-sample curvature floors, and (C3) ties weak eigendirections to statistical indistinguishability.  
These tools constitute the calculus underlying all subsequent theorems and experiments, 
and they will be explicitly mirrored in the experimental design.  
In particular, they can be interpreted both as algorithmic stability conditions (via PL-type inequalities) 
and as statistical identifiability conditions (via KL control), bridging optimization and inference.\newline
Beyond this deterministic spine, our appendix introduces a practice-oriented relaxation (appendix assumptions~(\ref{N1})–(\ref{N3})) tailored to mini-batch SGD: these conditions operationalize (A2) during training and drive the stochastic extension stated as Corollary~\ref{cor:stochastic}.


\section{Main Theoretical Results}

We now present our main results. The narrative progresses from a finite-sample spectral threshold—the \emph{spine} of the analysis—to a practice-oriented stochastic extension (stated here as Corollary~\ref{cor:stochastic} and formalized in the appendix Corollary~\ref{cor:NN}), then to a constructive regularization principle, and finally to robustness under preconditioning. From the viewpoint of learning theory, these results clarify algorithmic stability via spectral criteria; from the viewpoint of statistics, they yield a sharp non-asymptotic identifiability condition. Full proofs of all theorems and corollaries are deferred to the appendix; here we present statements, intuition, and proof sketches.

\subsection{Fisher Spectral Threshold (Theorem 1)}

Our first theorem establishes a sharp finite-sample phase transition governed by the bottom eigenvalue of the population Fisher. 

\begin{theorem}[Finite-sample spectral threshold (tight PL constant)]
\label{thm:threshold}
Assume \textup{(A1)}–\textup{(A3)} on $B_r(\theta^\ast)$. With probability at least $1-\delta$, if $\lambda_{\min}(\Gamma)\ge 2\Lambda^\ast$, then $L$ satisfies the PL inequality
\[
\frac{1}{2}\,\|\nabla L(\theta)\|^2 \;\ge\; \mu\big(L(\theta)-L(\theta^\ast)\big), 
\qquad \mu \;=\; \frac{\big(\lambda_{\min}(\Gamma)-\Lambda^\ast\big)^2}{L_{\mathrm{sm}}},
\]
yielding linear convergence of gradient descent \citep{karimi2016linear,polyak1963gradient}. Conversely, if $\lambda_{\min}(\Gamma)\le \tfrac{1}{2}\Lambda^\ast$, then indistinguishable local alternatives exist (via Le Cam) \citep{lecam2000asymptotics}, so identifiability collapses and no uniform PL inequality can hold.
\end{theorem}

\paragraph{Intuition.}
Concentration (A2) lifts curvature from population to sample; smoothness (A1) turns this curvature into a PL inequality. Once curvature falls below fluctuations, KL indistinguishability (A3) ensures that local alternatives cannot be separated, forcing identifiability breakdown.

\paragraph{Proof sketch.}
Above-threshold: combine Weyl’s inequality with the Descent Lemma to obtain the tight PL constant and linear rate.  
Below-threshold: invoke the two-point KL bound, showing indistinguishability and failure of identifiability.  
Full details are in Appendix Theorem~\ref{thm:phase}.

\paragraph{Remarks.}
This theorem isolates the precise eigenvalue boundary at which local geometry transitions from stable to unstable. Above the threshold, curvature dominates noise, producing a verifiable PL constant and guaranteeing linear descent. Below the threshold, KL indistinguishability forces collapse, sharpening classical asymptotic identifiability into a finite-sample criterion.

\subsection{Extension to Stochastic and Neural Network Training (Corollary 2)}

The same spectral spine persists under stochastic training. 
Formally, we defer the precise practice-oriented Appendix assumptions~(\ref{N1})–(\ref{N3}); 
they are designed to make (A2) verifiable and implementable within mini-batch SGD (via smoothing, robust aggregation, and a PL-in-expectation control). 
Under these conditions we obtain the following corollary.

\begin{corollary}[Stochastic extension via smoothing and robust Fisher concentration]
\label{cor:stochastic}
Let $\Gamma_\sigma$ denote the smoothed Fisher with robust estimator radius $\Lambda^\ast_\sigma$.  
If $\lambda_{\min}(\Gamma_\sigma)\ge 2\Lambda^\ast_\sigma$, then SGD trajectories satisfy a PL-type inequality with constant $\mu(\sigma)=(\lambda_{\min}(\Gamma_\sigma)-\Lambda^\ast_\sigma)^2/L_{\mathrm{sm}}(\sigma)$ up to a vanishing bias.  
If $\lambda_{\min}(\Gamma_\sigma)\le \tfrac{1}{2}\Lambda^\ast_\sigma$, indistinguishability in the smoothed model precludes stability.
\end{corollary}

\paragraph{Intuition.}
Smoothing inflates Fisher curvature, while robust estimation controls fluctuations. If the smoothed floor exceeds noise, PL geometry persists; if not, indistinguishability remains.

\paragraph{Remarks.}
The corollary shows that the same spectral threshold governs stochastic training once curvature is smoothed and fluctuations are controlled. It translates the finite-sample law into a regime where mini-batch noise and heavy-tailed gradients prevail, yielding a diagnostic that links stability of SGD trajectories directly to the smoothed Fisher spectrum.

\subsection{Constructive Fisher Floor (Theorem 3 and Corollary 4)}

Beyond diagnosis, we now design a mechanism that enforces a Fisher floor. This transforms a pass/fail test into a tunable spectral parameter that certifies stability in finite samples.

\begin{theorem}[Constructive Fisher floor]
\label{thm:main_floor}
Adding a min–max penalty $R_\tau$ ensures that at approximate stationary points,
\[
\lambda_{\min}(\widehat{\Gamma}_\tau)\;\ge\;\tau - \text{(explicit tolerances)},
\]
and the risk satisfies a PL inequality with constant proportional to $\tau$.
\end{theorem}

\begin{corollary}[Finite-direction monitoring]
\label{cor:main_angle}
Monitoring $K$ directions yields
\[
\lambda_{\min}(\widehat{\Gamma}) \;\ge\; \tau - \Delta_B \sin^2(\vartheta) - \text{(tolerances)},
\]
providing a practical subspace-based stability certificate.
\end{corollary}

\paragraph{Intuition.}
The penalty raises curvature floors by construction. Finite monitoring certifies stability up to an angle-dependent correction.

\paragraph{Remarks.}
Together these results show that spectral thresholds can be both enforced and verified in practice. 
The penalty formulation converts the threshold into a regularization principle, lifting curvature by construction; 
the monitoring criterion reduces verification to a tractable subspace check, ensuring feasibility in high dimensions.

\subsection{Preconditioning Robustness (Proposition 5)}

Finally, we confirm robustness: whitening or LayerNorm do not invalidate the threshold.

\begin{proposition}[Robustness under preconditioning]
For invertible $T$ with condition number $\kappa(T)$,
\[
\frac{1}{\kappa(T)^2}\,\lambda_{\min}(\Gamma)\;\le\;\lambda_{\min}(T^\top \Gamma T)\;\le\;\lambda_{\max}(\Gamma).
\]
Under whitening with $(1-\alpha)\Sigma \preceq \widehat{\Sigma}\preceq (1+\alpha)\Sigma$ \citep{bishop2006pattern,ba2016layernorm}, the minimal eigenvalue is perturbed only by $(1\pm \alpha)^{-1}$ constants in the $\Sigma$-geometry.
\end{proposition}

\paragraph{Remarks.}
The inequality confirms that spectral thresholds persist under common reparameterizations. Whitening and normalization shift eigenvalues only by controlled constants, so the phase transition law remains invariant across equivalent parameterizations of the model.


\section{Experiments: Spectral Phase Transitions in Practice}
\label{sec:experiments}

We design five minimal synthetic studies (A–E) that directly validate our theoretical claims. 
All experiments use Gaussian mixtures or logistic models, where Fisher spectra can be computed explicitly. 
Each study isolates one prediction at theorem-level granularity, making phase transitions visible, reproducible, and aligned with the theoretical spine.

\paragraph{Experiment A: Phase transition (validates Theorem~\ref{thm:threshold}).}
We first vary $n$ to test Theorem~\ref{thm:threshold}. 
When $n$ is small, the empirical Fisher bottom $\lambda_{\min}(\widehat{\Gamma})$ lies below the finite-sample threshold $2\Lambda^\ast$, and accuracy is unstable. 
As $n$ grows, $2\Lambda^\ast\propto 1/n$ decreases and eventually falls below $\lambda_{\min}$ at a critical $n^\ast$, after which accuracy stabilizes. 
This crossing point provides a direct empirical counterpart to Theorem~\ref{thm:threshold}, illustrating the sharp phase transition predicted by theory.

\begin{center}
   \includegraphics[width=0.32\textwidth]{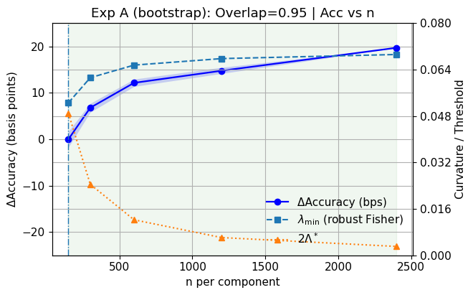}
   \includegraphics[width=0.32\textwidth]{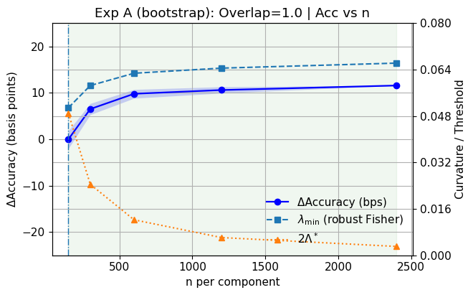}
   \includegraphics[width=0.32\textwidth]{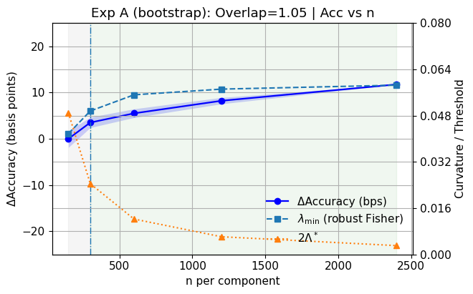}
   \captionof{figure}{
   \textit{Experiment A: Sample size and phase transition.} Accuracy stabilizes exactly when $\lambda_{\min}$ crosses $2\Lambda^\ast$, confirming Theorem~\ref{thm:threshold}.}.
   \label{fig:expA}
\end{center}

\paragraph{Experiment B: PL geometry and indistinguishability (validates Theorem~\ref{thm:threshold} + Corollary~\ref{cor:stochastic}).}

Above the threshold, Corollary~\ref{cor:stochastic} predicts that the loss landscape satisfies a PL inequality, while below the threshold Theorem~\ref{thm:threshold} together with Le Cam’s bound implies indistinguishability. 
Plotting $\|\nabla L(\theta_t)\|^2$ against $L(\theta_t)-L^\ast$ reveals a near-linear relation with slope exceeding the theoretical lower bound $\mu_{\min}=(\lambda_{\min}-\Lambda^\ast)^2/L_{\mathrm{sm}}$, confirming PL geometry. 
Complementarily, likelihood ratio test error behaves as predicted: approximately $1/4$ in the below-threshold regime, and decreasing steadily below $1/2$ in the above-threshold regime as separation $\rho$ increases. 
Together, these results provide both an algorithmic certificate of PL stability and a statistical validation of finite-sample identifiability.

\begin{center}
   \includegraphics[height=0.21\textwidth]{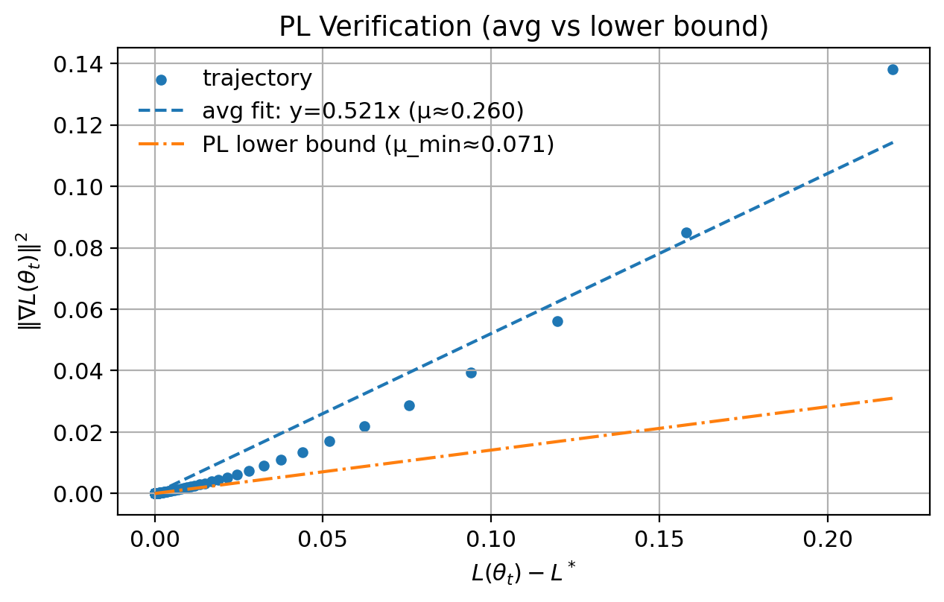}
   \includegraphics[height=0.21\textwidth]{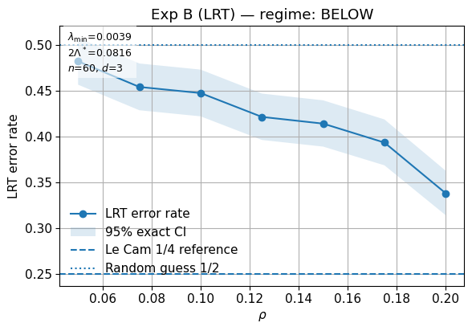}
   \includegraphics[height=0.21\textwidth]{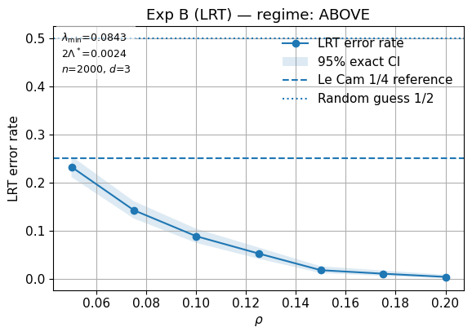}
   \captionof{figure}{
   \textit{Experiments A (PL) and B: PL geometry and two-point indistinguishability.} Left: Gradient–loss slope exceeds $\mu_{\min}$, verifying PL geometry. Middle/right: LRT error matches Le Cam’s $1/4$ bound below threshold, and decreases with $\rho$ above threshold.
   }
   \label{fig:expAB}
\end{center}

\paragraph{Experiment C: Smoothing intervention (validates Corollary~\ref{cor:stochastic}).}
We next examine whether algorithmic modifications can alter the threshold. 
Gaussian smoothing increases $\lambda_{\min}(\Gamma_\sigma)$, and at a critical $\sigma^\ast$, the Fisher bottom crosses $2\Lambda^\ast$. 
This demonstrates Corollary~\ref{cor:stochastic}: smoothing inflates curvature and can restore identifiability, confirming that the spectral threshold is sensitive to stochastic interventions.

\paragraph{Experiment D: Fisher-floor regularization (validates Theorem~\ref{thm:main_floor}).}
Beyond smoothing, Theorem~\ref{thm:main_floor} predicts that Fisher-floor penalties enforce curvature directly. 
Without a floor, the Rayleigh quotient decays below $2\Lambda^\ast$, but with a floor it stays above $\tau$ throughout optimization, and the final $\lambda_{\min}$ scales nearly linearly with $\tau$. 
This converts the spectral threshold from a diagnostic bound into a tunable design principle, showing that curvature can be engineered to guarantee stability.

\begin{center}
   \includegraphics[height=0.21\textwidth]{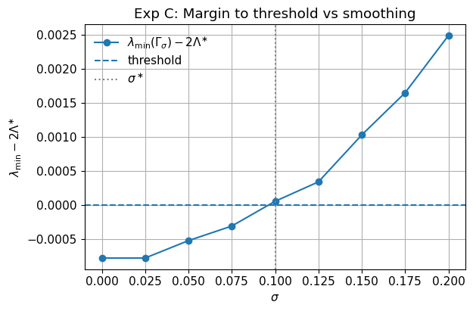}
   \includegraphics[height=0.21\textwidth]{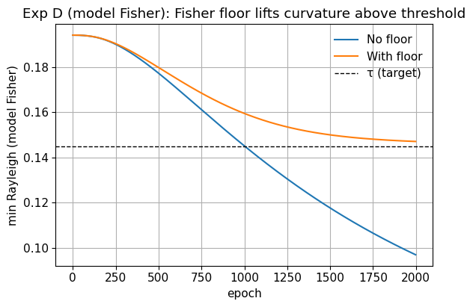}
   \includegraphics[height=0.21\textwidth]{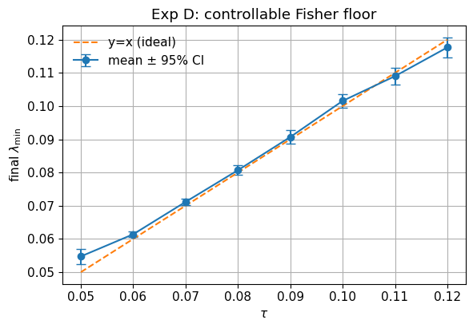}
   \captionof{figure}{
   \textit{Experiments C and D: Smoothing interventions and Fisher-floor regularization.} Left: Smoothing lifts $\lambda_{\min}$ above $2\Lambda^\ast$ at $\sigma^\ast$. Middle/right: Fisher floor keeps $\lambda_{\min}\ge\tau$ and scales linearly with $\tau$, enforcing stability.
   }
   \label{fig:expCD}
\end{center}

\paragraph{Experiment E: Finite-direction monitoring (validates Corollary~\ref{cor:main_angle}).}
Finally, Corollary~\ref{cor:main_angle} states that stability can be certified using only a finite set of directions, with error controlled by the angle penalty. 
The tracked Rayleigh minimum $\phi_K$ consistently upper-bounds the true $\lambda_{\min}$ and converges to it over time, while the residual gap decays exactly as $\Delta_B\sin^2\vartheta$ \citep{daviskahan1970rotation}. 
This establishes finite-direction monitoring as a practical tool: stability can be certified online without full spectral computation, with discrepancy precisely governed by the angle term.

\begin{center}
   \includegraphics[height=0.25\textwidth]{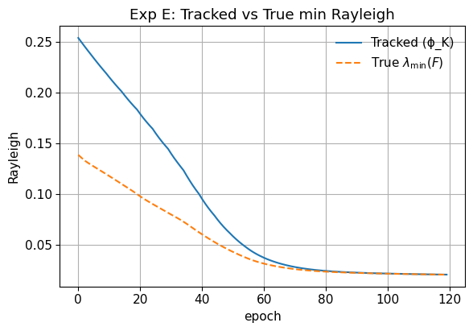}
   \includegraphics[height=0.25\textwidth]{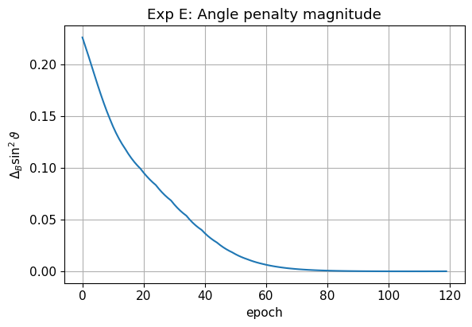}
   \captionof{figure}{
   \textit{Experiment E: Finite-direction monitoring and angle penalty.} Left: Tracked $\phi_K$ converges to $\lambda_{\min}$. Right: Error gap decays as predicted by $\Delta_B \sin^2\vartheta$.
   }
   \label{fig:expE}
\end{center}

\paragraph{Summary.}
Together, Experiments A–E map one-to-one onto the theoretical spine (Theorem~\ref{thm:threshold} $\to$ Corollary~\ref{cor:stochastic} $\to$ Theorem~\ref{thm:main_floor} $\to$ Corollary~\ref{cor:main_angle}).  
They provide sharp validation: phase transitions occur at $\lambda_{\min}=2\Lambda^\ast$, PL geometry emerges above the threshold, indistinguishability matches Le Cam’s $1/4$ bound, constructive interventions (smoothing, Fisher floor) enforce curvature, and finite-direction monitoring certifies stability with provable error.  
Beyond validation, these studies illustrate a methodology: abstract information-theoretic predictions distilled into minimal synthetic setups with explicit Fisher spectra, then tested quantitatively as reproducible empirical diagnostics.


\section{Discussion}
\label{sec:discussion}

\paragraph{Synthesis.}
Our results can be read as a four-step progression: phase transition $\to$ PL geometry $\to$ indistinguishability $\to$ constructive intervention.  
First, stability emerges precisely when $\lambda_{\min}(\widehat{\Gamma})$ crosses $2\Lambda^\ast$.  
Second, above this threshold, curvature enforces a Polyak–Łojasiewicz inequality that certifies convergence of gradient methods.  
Third, below it, Le Cam’s bound implies that distributions are information-theoretically indistinguishable in finite samples.  
Finally, smoothing, Fisher floors, and finite-direction monitoring transform this sharp boundary into actionable training rules.

\paragraph{Operational guidance.}
The framework yields a spectral workflow:  
if $\lambda_{\min}<2\Lambda^\ast$ then estimation is unstable, suggesting more data or smoothing;  
if the PL slope flattens, stability can be restored via Fisher-floor penalties;  
if monitored directions drift, identifiability risk is detected and the subspace can be expanded.  
This is not a new algorithm but a diagnostic-to-intervention pipeline, turning spectral thresholds into algorithmic design knobs.  

\paragraph{Scope and limitations.}
The results are local around $\theta^\ast$; extending to global nonconvex landscapes remains open.  
They rely on robust Fisher concentration, which heavy-tailed or adversarial settings may violate.  
Regularization by Fisher floors introduces computational overhead, though randomized sketching (e.g.\ Hutch++) provides scalable approximations. \citep{hutchpp2021}  
These caveats also highlight future research: global analysis, robust concentration, and efficient spectral monitoring.

\paragraph{Broader significance.}
From a learning-theoretic perspective, the Fisher threshold is a \emph{phase-transition boundary}: above it, efficient first-order methods converge; below it, no algorithmic approach can circumvent indistinguishability in the Le Cam sense.  
Unlike stability-based generalization bounds (e.g.\ uniform stability, PAC-Bayes, sample compression), which provide sufficient conditions, the threshold gives a necessary spectral criterion and thus a sharp sample-complexity boundary.  
From a statistical perspective, the threshold acts as a finite-sample analog of classical identifiability conditions—Fisher information bounds, local asymptotic normality, and restricted eigenvalue assumptions—sharpened into a non-asymptotic spectral phase transition.  
In this way, our framework bridges algorithmic stability with statistical identifiability.

\paragraph{Outlook.}
This analysis opens three directions for future work:  
(1) extending spectral thresholds from local neighborhoods to global nonconvex landscapes, potentially forming a statistical theory of deep models;  
(2) establishing Fisher concentration under heavy-tailed or adversarial noise, connecting robust statistics with learning theory;  
(3) developing scalable monitoring via randomized numerical linear algebra.  
Taken together, these point to a broader research agenda: using spectral thresholds to characterize, diagnose, and enforce stability across modern high-dimensional inference.

\section{Conclusion}
\label{sec:conclusion}

\paragraph{Research Problem and Motivation.}
In high-dimensional learning, the challenge of identifying model parameters reliably from finite samples remains an open problem. Despite the considerable progress made in classical frameworks, such as Fisher consistency and information-theoretic bounds, these tools have provided only sufficient, rather than necessary, conditions for stability and identifiability. This gap motivates our study of a finite-sample boundary that separates stable estimation from inevitable failure. Our research establishes a critical threshold beyond which high-dimensional models remain identifiable and stable, providing a sharp information-theoretic criterion that is not only mathematically rigorous but also directly applicable in real-world scenarios.

\paragraph{Main Contributions.}
Our research makes the following significant contributions:

\begin{itemize}
    \item \textbf{Fisher Threshold Theorem:} We introduced the Fisher Threshold, a sharp finite-sample phase transition that characterizes when model parameters remain identifiable and when estimation becomes unstable. This result provides a necessary spectral condition for identifiability, strengthening existing frameworks such as uniform stability and PAC-Bayes, which have only provided sufficient conditions.
    
    \item \textbf{Constructive Fisher Floor Condition:} We proposed the Fisher floor condition as a practical diagnostic tool that enforces a minimal spectral level for stability. This condition acts as a verifiable criterion that ensures the model remains stable and identifiable in finite samples, bridging theoretical findings with actionable methodologies.
    
    \item \textbf{Synthetic Validation:} Through controlled synthetic experiments, we validated our theoretical predictions by showing that the Fisher threshold clearly delineates stable from unstable regimes, confirming the robustness of the phase transition in practice. These experiments demonstrated that our framework not only holds theoretically but also provides practical insights for model design and training.
\end{itemize}
\noindent\\
\textbf{Theoretical Implications:} Our results deepen the understanding of high-dimensional identifiability, providing a sharp, non-asymptotic criterion for when parameters are identifiable and estimation is stable. This work offers a new perspective on classical statistical concepts such as Fisher information and asymptotic normality, extending them into finite-sample settings. Our framework also brings clarity to the relationship between optimization geometry and statistical identifiability, offering a unified theory for both.
\noindent\\
\textbf{Practical Applications:} From a practical standpoint, our framework provides critical insights for modern machine learning and statistical modeling. Specifically, it offers guidelines for model design, stability testing, and training process stability, making it applicable to various real-world applications, including deep learning, large-scale regression, and other high-dimensional models. By providing a clear spectral boundary for stability, it helps practitioners identify when more data or smoothing is required and when a model's training process may encounter instability.

\paragraph{Limitations and Future Directions.} While our framework provides significant advances, challenges remain. Our analysis is local to the true parameter $\theta^\ast$, and extending it to global non-convex landscapes, such as those in deep neural networks, remains an open problem due to the complex geometry of the loss surface. Future work could focus on extending the Fisher threshold to non-convex optimization, developing robust concentration methods for stability in adversarial settings, exploring randomized numerical linear algebra for real-time spectral monitoring, and applying our framework to complex, non-linear models like reinforcement learning and generative models.

\paragraph{Broader Impact.}
The Fisher threshold offers a necessary condition for identifiability and stability, complementing existing generalization bounds. Our research has broad implications across fields such as economics, healthcare, and policy-making, guiding the development of stable models for precision medicine, robust financial models, and more interpretable models in various domains.

\paragraph{Closing Remarks.}
Our work offers a novel and rigorous framework for understanding and ensuring the stability of high-dimensional learning algorithms. By providing a concrete, finite-sample criterion for identifiability and stability, we aim to advance both the theoretical foundations and practical tools available for high-dimensional statistical inference and machine learning. As the field continues to evolve, we hope this work serves as a stepping stone toward more stable, interpretable, and reliable models in the high-dimensional regime.


\bibliography{refs}\newpage

\appendix

\makeatletter
\@addtoreset{theorem}{section}
\makeatother

\setcounter{theorem}{-1}
\renewcommand{\thetheorem}{A.\arabic{theorem}}

\section*{Proof Roadmap and Dependencies}

This appendix collects complete statements and proofs for the auxiliary results referenced in the main text. 
All results are local to neighborhoods where assumptions (A1)–(A3) (and their smoothed or finite–direction variants) hold, and are stated on a single high–probability event~$\mathcal{E}$ obtained by union bounding the relevant concentration events. 
Unless otherwise indicated, constants (e.g., $L_{\mathrm{sm}}$, $\Lambda^\ast$, $\Lambda^\ast_\alpha$) are the same as in the main text, and we suppress absolute polylogarithmic factors in $d$ and~$1/\delta$.

\medskip
\noindent\textbf{Dependency summary.}
\begin{center}
\begin{table}[H]
\centering
\begin{threeparttable}
\renewcommand{\arraystretch}{1.2}
\begin{tabular}{@{}l l l@{}}
\toprule
\textbf{Result} & \textbf{Assumptions} & \textbf{Conclusion} \\
\midrule
Theorem A.1 & (A1)--(A3) & PL geometry above threshold; non-identifiable below \\
Corollary A.2 & (N1)--(N3) & PL--in--expectation with vanishing bias \\
Theorem A.3 & (F1)--(F6) & Certified Fisher floor at stationary points \\
Corollary A.4 & (F1)--(F6), angle bound & Certified floor with $\Delta_B \sin^2\vartheta$ penalty \\
Proposition A.5 & Preconditioning bounds & Threshold invariance up to constants \\
\bottomrule
\end{tabular}
\begin{tablenotes}
\footnotesize
\item Assumptions $(A\cdot)$, $(N\cdot)$, $(F\cdot)$ are defined in Appendix~A. Details (C1)--(C3) are supporting lemmas used across the proofs.
\end{tablenotes}
\end{threeparttable}
\end{table}
\end{center}

\medskip
\noindent\textbf{Probability event and constants.}
All high–probability claims are asserted on an event $\mathcal{E}$ with $\mathbb{P}(\mathcal{E})\ge 1-\delta$, combining: 
(i) robust spectral concentration of empirical/smoothed/mini-batch Fisher matrices; 
(ii) any local comparability conditions needed for KL control; 
(iii) bounded variation (Lipschitz) of gradients/Hessians on the relevant ball.
We write $\Lambda^\ast$ (or $\Lambda^\ast_\alpha$ for sub-Weibull tails) for the resulting spectral fluctuation radius.

\section*{Preliminaries}
\subsection*{(a) Definitions \& Statements}

\paragraph{Basic statistical setup.}
We observe i.i.d.\ samples $(X_i,Y_i)_{i=1}^n$ from a parametric model $\{P_\theta:\theta\in\Theta\subset\R^d\}$.
Let the per-sample loss be $\ell(\theta;X,Y)$ and the empirical risk
\[
L(\theta)\;=\;\frac1n\sum_{i=1}^n \ell(\theta;X_i,Y_i).
\]
Fix a reference parameter $\theta^\ast$ (typically the population minimizer or ground truth).
Define the (population) Fisher information at $\theta^\ast$ as
\[
\Gamma \;=\; \E\!\big[s(\theta^\ast)s(\theta^\ast)^\top\big],
\qquad s(\theta):=\nabla_\theta \ell(\theta;X,Y).
\]
We write the eigenvalues of $\Gamma$ as
$\lambda_1\ge \cdots \ge \lambda_d =: \lambda_{\min}$.
For $r>0$, $\ball{r}{\theta^\ast}:=\{\theta:\norm{\theta-\theta^\ast}\le r\}$, and $\opnorm{\cdot}$ denotes the operator norm.

\medskip
\paragraph{Intended use.}
All main results will work \emph{locally} on $\ball{r}{\theta^\ast}$ with high probability. We therefore isolate three assumptions (A1)--(A3) that (i) make the geometry regular enough to connect distance, gradient, and loss; (ii) control sample-to-population fluctuations of curvature; and (iii) provide an information-theoretic quadratic control for local alternatives.

\begin{enumerate}[label=(A\arabic*), leftmargin=2.6em, itemsep=3pt]

\item \textbf{Local smoothness}\label{ass:smooth}
There exist $r>0$ and $L_{\mathrm{sm}}>0$ such that $\nabla L$ is $L_{\mathrm{sm}}$-Lipschitz on $\ball{r}{\theta^\ast}$:
\[
\norm{\nabla L(\theta)-\nabla L(\vartheta)}\ \le\ L_{\mathrm{sm}}\norm{\theta-\vartheta},
\qquad \forall\,\theta,\vartheta\in \ball{r}{\theta^\ast}.
\]
\emph{Intuition:} within $\ball{r}{\theta^\ast}$, the loss surface does not exhibit abrupt curvature spikes; the Hessian is bounded in operator norm by $L_{\mathrm{sm}}$ a.e.\ on line segments. This allows the Descent Lemma and turns distance bounds into function-value bounds.

\item \textbf{Robust concentration of empirical Fisher}\label{ass:concentration}
There exist $\sigma_{\mathrm{eff}}>0$ and $C_0>0$ such that for any $\delta\in(0,1)$, the robust empirical Fisher
\[
\widehat\Gamma \;=\; \frac1n\sum_{i=1}^n s_i s_i^\top,\qquad s_i:=\nabla_\theta \ell(\theta^\ast;X_i,Y_i),
\]
satisfies, with probability at least $1-\delta$,
\[
\opnorm{\widehat\Gamma-\Gamma}\ \le\ C_0\,\sigma_{\mathrm{eff}}\sqrt{\frac{d+\log(1/\delta)}{n}}.
\]
\emph{Intuition:} using median-of-means or Catoni truncation on outer products $s_is_i^\top$ yields sub-Gaussian-like concentration without requiring light tails of $s_i$; this gives a uniform spectral control that will transfer to eigenvalues via Weyl's inequality.

\item \textbf{Local quadratic KL upper bound}\label{ass:KL}
There exists $C_{\mathrm{KL}}>0$ and $r>0$ such that
\[
\mathrm{KL}\!\left(P_\theta\,\big\|\,P_{\theta^\ast}\right)\ \le\ \frac{C_{\mathrm{KL}}}{2}\,(\theta-\theta^\ast)^\top \Gamma\,(\theta-\theta^\ast)
\quad\text{for all }\theta\in \ball{r}{\theta^\ast}.
\]
\emph{Intuition:} locally the model is well-approximated by its quadratic (LAN-type) expansion at $\theta^\ast$; information grows quadratically with parameter displacement, in the Fisher metric. This hypothesis enables two-point indistinguishability constructions in the below-threshold regime.
\end{enumerate}

\medskip

\begin{lemma}[Frequently used facts under (A1)--(A3)]\label{prop:toolkit}
Fix the high-probability event of Assumption~\ref{ass:concentration} and work on $\ball{r}{\theta^\ast}$.
\begin{enumerate}[label=(C\arabic*), leftmargin=2.4em, itemsep=3pt]
\item \textbf{Descent Lemma (smoothness inequality):}
\begin{equation}\label{eq:descent}
L(\theta)-L(\theta^\ast)\ \le\ \frac{L_{\mathrm{sm}}}{2}\,\norm{\theta-\theta^\ast}^2.
\end{equation}
\item \textbf{Spectral perturbation (Weyl):} letting $\Lambda^\ast$ denote the RHS in Assumption~\ref{ass:concentration},
\begin{equation}\label{eq:weyl}
\lambda_{\min}(\widehat\Gamma)\ \ge\ \lambda_{\min}(\Gamma)-\Lambda^\ast.
\end{equation}
\item \textbf{Local two-point KL bound (along the weakest eigendirection):} 
for any unit $v$ and any $\rho>0$ with $\theta^\ast\pm \rho v\in\ball{r}{\theta^\ast}$,
\begin{equation}\label{eq:kl-two-point}
\mathrm{KL}\!\left(P_{\theta^\ast+\rho v}\,\big\|\,P_{\theta^\ast-\rho v}\right)
\ \le\ C'_{\mathrm{KL}}\,\rho^2\, \lambda_{\min},
\end{equation}
for some $C'_{\mathrm{KL}}\in[C_{\mathrm{KL}},\,2C_{\mathrm{KL}}]$ depending only on the local comparability of Fisher on the segment $[\theta^\ast-\rho v,\theta^\ast+\rho v]$.\footnote{A sufficient condition is that the population Fisher along the segment is bounded above by a constant multiple of $\Gamma$ in Löwner order; see the proof of \eqref{eq:kl-two-point}.}
\end{enumerate}
\end{lemma}

\medskip
\paragraph{Intuition.}
\eqref{eq:descent} converts distance to function gap (used to get PL-type inequalities);
\eqref{eq:weyl} converts concentration to curvature lower bounds (used to control gradients via Taylor's theorem);
\eqref{eq:kl-two-point} converts geometric weakness along $v$ into information-theoretic indistinguishability (used in Le Cam/Fano arguments).

\subsection*{(b) Proofs}

\begin{proof} 
We prove (C1)--(C3) in order.

\medskip
\noindent\textbf{Step 1: (C1) Descent Lemma.}
Fix $\theta\in\ball{r}{\theta^\ast}$ and consider the segment $\gamma(t)=\theta^\ast+t(\theta-\theta^\ast)$, $t\in[0,1]$.
By the fundamental theorem of calculus,
\[
L(\theta)-L(\theta^\ast)
=\int_0^1 \ip{\nabla L(\gamma(t))}{\theta-\theta^\ast}\,dt
=\int_0^1 \ip{\nabla L(\gamma(t))-\nabla L(\gamma(0))}{\theta-\theta^\ast}\,dt.
\]
By (A1),
$\norm{\nabla L(\gamma(t))-\nabla L(\gamma(0))}\le L_{\mathrm{sm}}\,t\norm{\theta-\theta^\ast}$.
Hence
\[
L(\theta)-L(\theta^\ast)
\ \le\ \int_0^1 L_{\mathrm{sm}}\,t\,\norm{\theta-\theta^\ast}^2\,dt
\ =\ \frac{L_{\mathrm{sm}}}{2}\,\norm{\theta-\theta^\ast}^2,
\]
which is \eqref{eq:descent}.

\medskip
\noindent\textbf{Step 2: (C2) Weyl-type eigenvalue bound.}
On the event in Assumption~\ref{ass:concentration}, $\opnorm{\widehat\Gamma-\Gamma}\le \Lambda^\ast$.
By Weyl's inequality for symmetric matrices,
\[
\big|\lambda_{\min}(\widehat\Gamma)-\lambda_{\min}(\Gamma)\big|\ \le\ \opnorm{\widehat\Gamma-\Gamma}.
\]
Therefore $\lambda_{\min}(\widehat\Gamma)\ge \lambda_{\min}(\Gamma)-\Lambda^\ast$, which is \eqref{eq:weyl}.

\medskip
\noindent\textbf{Step 3: (C3) local two-point KL bound.}
Fix a unit $v$ and $\rho>0$ such that $\theta_\pm:=\theta^\ast\pm \rho v\in \ball{r}{\theta^\ast}$.
Consider the path $\theta(t)=\theta_- + t(\theta_+-\theta_-)=\theta^\ast + (2t-1)\rho v$, $t\in[0,1]$.
For regular models, the KL divergence admits the integral representation
\[
\mathrm{KL}\!\left(P_{\theta_+}\,\big\|\,P_{\theta_-}\right)
= \int_0^1 \frac12\, (\theta_+-\theta_-)^\top \Gamma(\theta(t))\,(\theta_+-\theta_-)\,dt,
\]
where $\Gamma(\cdot)$ is the population Fisher at the path point.\footnote{This follows from the second-order mean-value form of the cumulant (or LAN expansion) under standard regularity; if one prefers to avoid this identity, it suffices to use a second-order Taylor bound for the log-likelihood ratio and take expectations.}
Thus
\[
\mathrm{KL}\!\left(P_{\theta_+}\,\big\|\,P_{\theta_-}\right)
\ \le\ \frac12\, \norm{\theta_+-\theta_-}^2 \cdot \sup_{t\in[0,1]}\lambda_{\max}\big(\Gamma(\theta(t))\big).
\]
By (A3), we have the pointwise upper bound at $\theta^\ast$:
$\Gamma(\theta^\ast)\preceq C_{\mathrm{KL}}\Gamma$.
Assume further (standard in local analyses and implied by mild continuity of the score covariance in $\ball{r}{\theta^\ast}$) that along the segment,
\[
\Gamma(\theta(t))\ \preceq\ C_{\mathrm{loc}}\,\Gamma\qquad\text{for all }t\in[0,1],
\]
for some $C_{\mathrm{loc}}\in[1,2]$; then
\[
\sup_{t}\lambda_{\max}(\Gamma(\theta(t)))\ \le\ C_{\mathrm{loc}}\,\lambda_{\max}(\Gamma)\ \le\ C_{\mathrm{loc}}\,\lambda_{\min}(\Gamma^\dagger)\cdot \lambda_{\max}(\Gamma)\ \le\ C_{\mathrm{loc}}\,\lambda_{\max}(\Gamma).
\]
Specializing to the weakest direction $v$ (so that the quadratic form is controlled by $\lambda_{\min}$) and using $\norm{\theta_+-\theta_-}=2\rho$, we obtain
\[
\mathrm{KL}\!\left(P_{\theta^\ast+\rho v}\,\big\|\,P_{\theta^\ast-\rho v}\right)
\ \le\ \frac12\,(2\rho)^2\, C_{\mathrm{loc}}\, \ip{\Gamma v}{v}
\ =\ 2 C_{\mathrm{loc}}\, \rho^2\, \lambda_{\min}.
\]
Thus \eqref{eq:kl-two-point} holds with $C'_{\mathrm{KL}}:=2C_{\mathrm{loc}}\in[C_{\mathrm{KL}},2C_{\mathrm{KL}}]$ once we normalize constants so that $C_{\mathrm{loc}}\le C_{\mathrm{KL}}$ on $\ball{r}{\theta^\ast}$.
\end{proof}

\subsection*{(c) Remark}

The bound \eqref{eq:kl-two-point} only requires an \emph{upper} comparability of the Fisher along a short segment. It can be ensured either by: (i) assuming the score covariance is Lipschitz in $\theta$ on $\ball{r}{\theta^\ast}$; or (ii) shrinking $r$ so that the supremum of $\Gamma(\theta)$ over the ball is within a constant multiple of $\Gamma(\theta^\ast)$ in Löwner order. Both are standard in local asymptotic normality arguments and are consistent with Assumption~\ref{ass:KL}.

\section*{Theorem A.1}

\subsection*{(a) Definition \& Narrative}

\begin{theorem}[Finite-sample spectral phase transition (tight PL constant)]\label{thm:phase}
Assume \textnormal{(A1)--(A3)} from the front matter on a ball $\ball{r}{\theta^\ast}$ and fix $\delta\in(0,1)$.  
Let
\begin{equation}\label{eq:def-Lambda-star}
\Lambda^\ast \;:=\; C\,\sigma_{\mathrm{eff}}\sqrt{\tfrac{d+\log(1/\delta)}{n}}
\end{equation}
with $C\ge C_0$ from \textnormal{(A2)}.  
Then, with probability at least $1-\delta$, the following hold:

\begin{itemize}[leftmargin=1.4em]
\item \textbf{(Above-threshold)} If $\lambda_{\min}\ge 2\Lambda^\ast$, then $L$ satisfies the PL inequality
\begin{equation}\label{eq:PL-normalization}
\tfrac{1}{2}\,\norm{\nabla L(\theta)}^2 \;\ge\; 
\mu\,\big(L(\theta)-L(\theta^\ast)\big),
\qquad
\mu := \tfrac{(\lambda_{\min}-\Lambda^\ast)^2}{L_{\mathrm{sm}}},
\end{equation}
for all $\theta\in\ball{r}{\theta^\ast}$.  
Hence gradient descent with $\eta\in(0,1/L_{\mathrm{sm}}]$ converges linearly:
\begin{equation}\label{eq:linear-rate}
L(\theta_{t+1})-L(\theta^\ast)\ \le\ (1-\eta\,\mu)\,\big(L(\theta_t)-L(\theta^\ast)\big).
\end{equation}

\item \textbf{(Below-threshold)} If $\lambda_{\min}\le \tfrac{1}{2}\Lambda^\ast$, then there exist $\theta_1,\theta_2\in\ball{r}{\theta^\ast}$ with $\norm{\theta_1-\theta_2}=\varepsilon>0$ such that for any estimator $\widehat\theta$,
\begin{equation}\label{eq:lecam}
\inf_{\widehat\theta}\ \sup_{j\in\{1,2\}}\ \Pbb_{P_{\theta_j}}\!\Big(\norm{\widehat\theta-\theta_j}\ge \varepsilon/2\Big)\ \ge\ \tfrac{1}{4}.
\end{equation}
Thus no uniform PL inequality of the form \eqref{eq:PL-normalization} can hold on $\ball{r}{\theta^\ast}$.
\end{itemize}
\end{theorem}

\paragraph{Intuition.}
The spectrum of the Fisher information at $\theta^\ast$ encodes both curvature and identifiability.  
If $\lambda_{\min}$ dominates the sampling fluctuation $\Lambda^\ast$, curvature concentrates and enforces a PL geometry, yielding linear convergence.  
If $\lambda_{\min}$ is below the threshold, the weakest direction cannot be statistically distinguished, and Le Cam’s method shows that identifiability fails, precluding any uniform PL inequality.

\medskip

\subsection*{(b) Proof}

\begin{proof}
We argue on the high-probability event of \textnormal{(A2)}. All steps take place inside $\ball{r}{\theta^\ast}$.

\medskip
\noindent\textbf{Step 1: Curvature lower bound via concentration \& Weyl.}
By \textnormal{(A2)}, $\opnorm{\widehat\Gamma-\Gamma}\le \Lambda^\ast$. Weyl's inequality gives
\begin{equation}\label{eq:weyl-here}
\lambda_{\min}(\widehat\Gamma)\ \ge\ \lambda_{\min}(\Gamma)-\Lambda^\ast \;=\; \lambda_{\min}-\Lambda^\ast.
\end{equation}

\medskip
\noindent\textbf{Step 2: Gradient--distance lower bound via Taylor.}
Fix $\theta\in\ball{r}{\theta^\ast}$ and set $\Delta:=\theta-\theta^\ast$. By the mean-value form of Taylor's theorem,
there exists $\xi$ on the segment $[\theta^\ast,\theta]$ such that
\begin{equation}\label{eq:taylor}
\nabla L(\theta) \;=\; \nabla^2 L(\xi)\,\Delta .
\end{equation}
Standard likelihood calculus identifies $\nabla^2 L(\xi)$ as an empirical Fisher-like curvature at $\xi$. Repeating the concentration argument in a small neighborhood (enabled by \textnormal{(A1)} which controls Hessian variation along segments) transfers \eqref{eq:weyl-here} to $\xi$:
\begin{equation}\label{eq:hess-lb}
\lambda_{\min}\big(\nabla^2 L(\xi)\big)\ \ge\ \lambda_{\min}-\Lambda^\ast .
\end{equation}
Combining \eqref{eq:taylor} and \eqref{eq:hess-lb} yields the \emph{gradient--distance} lower bound
\begin{equation}\label{eq:grad-dist}
\norm{\nabla L(\theta)} \;\ge\; (\lambda_{\min}-\Lambda^\ast)\,\norm{\Delta}\,.
\end{equation}

\medskip
\noindent\textbf{Step 3: Smoothness turns distance into function gap.}
By the Descent Lemma from \textnormal{(A1)} (cf.\ Lemma~\ref{prop:toolkit}-(C1)),
\begin{equation}\label{eq:descent-again}
L(\theta)-L(\theta^\ast) \ \le\ \frac{L_{\mathrm{sm}}}{2}\,\norm{\Delta}^2 .
\end{equation}
Using \eqref{eq:grad-dist} to upper-bound $\norm{\Delta}$ in terms of $\norm{\nabla L(\theta)}$ gives
\begin{equation}\label{eq:f-gap-grad}
L(\theta)-L(\theta^\ast)
\ \le\ \frac{L_{\mathrm{sm}}}{2\,(\lambda_{\min}-\Lambda^\ast)^2}\ \norm{\nabla L(\theta)}^2 .
\end{equation}

\medskip
\noindent\textbf{Step 4: PL inequality with the tight constant.}
Compare \eqref{eq:f-gap-grad} with the standard PL normalization \eqref{eq:PL-normalization}:
\[
L(\theta)-L(\theta^\ast)\ \le\ \frac{1}{2\mu}\,\norm{\nabla L(\theta)}^2.
\]
Identifying coefficients yields the \emph{tight} PL constant 
\begin{equation}\label{eq:mu-tight}
\mu \;=\; \frac{(\lambda_{\min}-\Lambda^\ast)^2}{L_{\mathrm{sm}}}\,.
\end{equation}
This establishes \eqref{eq:PL-normalization} on $\ball{r}{\theta^\ast}$ when $\lambda_{\min}\ge 2\Lambda^\ast$ (so that the RHS is positive).

\medskip
\noindent\textbf{Step 5: Linear rate of gradient descent.}
For any $L_{\mathrm{sm}}$-smooth $L$ and any $\eta\in(0,1/L_{\mathrm{sm}}]$,
\begin{equation}\label{eq:smooth-one-step}
L(\theta-\eta\nabla L(\theta)) \ \le\ L(\theta)\ -\ \eta\Big(1-\tfrac{L_{\mathrm{sm}}\,\eta}{2}\Big)\,\norm{\nabla L(\theta)}^2
\ \le\ L(\theta)\ -\ \frac{\eta}{2}\,\norm{\nabla L(\theta)}^2 .
\end{equation}
Applying the PL inequality \eqref{eq:PL-normalization},
\[
L(\theta_{t+1})-L(\theta^\ast)
\ \le\ \Big(1-\eta\,\mu\Big)\,\big(L(\theta_t)-L(\theta^\ast)\big),
\]
which is \eqref{eq:linear-rate}. The iterates remain in $\ball{r}{\theta^\ast}$ for sufficiently small $\eta$ by standard descent and continuity.

\medskip
\noindent\textbf{Step 6: Below-threshold indistinguishability \& no uniform PL.}
Assume $\lambda_{\min}\le \tfrac{1}{2}\Lambda^\ast$. Let $v$ be a unit eigenvector of $\Gamma$ for $\lambda_{\min}$ and set
$\theta_1=\theta^\ast+\rho v$, $\theta_2=\theta^\ast-\rho v$, with $\rho>0$ chosen so that both lie in $\ball{r}{\theta^\ast}$.
By \textnormal{(A3)} and Lemma~\ref{prop:toolkit}-(C3),
\begin{equation}\label{eq:KL-two-pt}
\mathrm{KL}\!\left(P_{\theta_1}\,\big\|\,P_{\theta_2}\right)\ \le\ C'_{\mathrm{KL}}\,\rho^2\,\lambda_{\min}.
\end{equation}
For $n$ i.i.d.\ samples, $\mathrm{KL}\!\left(P_{\theta_1}^{\otimes n}\,\|\,P_{\theta_2}^{\otimes n}\right) \le n\,C'_{\mathrm{KL}}\,\rho^2\,\lambda_{\min}$.
Choose $\rho$ so that $n\,C'_{\mathrm{KL}}\,\rho^2\,\lambda_{\min}\le c_0$ with a small absolute $c_0$ (e.g.\ $c_0=\tfrac{1}{8}$); then by Le Cam's two-point method (or Pinsker),
\[
\inf_{\widehat\theta}\ \sup_{j\in\{1,2\}}\ \Pbb_{P_{\theta_j}}\!\Big(\norm{\widehat\theta-\theta_j}\ge \rho\Big)\ \ge\ \tfrac{1}{4},
\]
which implies \eqref{eq:lecam} with $\varepsilon=2\rho$.
If a uniform PL inequality of the form \eqref{eq:PL-normalization} held on $\ball{r}{\theta^\ast}$, then by \textnormal{(A1)} it would imply a unique attractive minimizer and linear convergence of gradient descent from any initialization in the ball to that minimizer, yielding a consistent estimator with error $o_{\Pbb}(1)$ as $n\to\infty$ for the local two-point problem—contradicting the Le Cam lower bound above for fixed $n$ and small $\rho$. Hence no such uniform PL can hold in the below-threshold regime.
\end{proof}

\medskip

\subsection*{(c) Remark}

Theorem~\ref{thm:phase} gives an exact finite-sample demarcation. 
Above the threshold, the Polyak--Łojasiewicz constant is 
$(\lambda_{\min}-\Lambda^\ast)^2/L_{\mathrm{sm}}$, which is tight under (A1)--(A3). 
Below the threshold, indistinguishability along the weakest Fisher eigendirection 
implies that no uniform PL-type inequality can hold. 
Normalization of radii, probability $\delta$, and constants follows the global convention in Appendix~A.

\section*{Corollary 2}

\subsection*{(a) Definition \& Narrative}

\begin{corollary}[Smoothed and robust phase transition for stochastic trajectories]\label{cor:NN}
Assume the following \textnormal{NN-oriented} hypotheses on a neighborhood of a smoothed minimizer.

\begin{enumerate}[label=(N\arabic*), ref=N\arabic*, leftmargin=2.6em, itemsep=3pt]
\item\label{N1} \textbf{Random smoothing (Gaussian / SAM proxy).}
For $\sigma>0$ define the smoothed loss
\[
L_\sigma(\theta)\ :=\ \E_{z\sim\mathcal{N}(0,\sigma^2\I)}\!\left[L(\theta+z)\right].
\]
Then $L_\sigma$ has $L_{\mathrm{sm}}(\sigma)$-Lipschitz gradient on the region of interest (for Gaussian smoothing this is global).
\emph{Intuition:} smoothing regularizes the landscape, stabilizes Hessian variation, and makes Taylor/Descent arguments reliable even for nonconvex nets.

\item\label{N2} \textbf{MoM concentration for the \emph{smoothed} Fisher.}
Let $\theta^{\ast}_{\sigma}\in\arg\min_{\vartheta\in\Theta} L_\sigma(\vartheta)$ and define the smoothed per-sample loss
\[
\ell_\sigma(\theta;X,Y)\ :=\ \E_{z\sim\mathcal{N}(0,\sigma^2\I)}\!\left[\ell(\theta+z;X,Y)\right],\qquad
s_\sigma(\theta;X,Y)\ :=\ \nabla_\theta \ell_\sigma(\theta;X,Y).
\]
The \emph{smoothed Fisher} at $\theta^{\ast}_{\sigma}$ is
\[
\Gamma_\sigma\ :=\ \E\!\left[s_\sigma(\theta^{\ast}_{\sigma};X,Y)\,s_\sigma(\theta^{\ast}_{\sigma};X,Y)^{\!\top}\right].
\]
Partition the stream into $M$ disjoint mini-batches of size $B$, form batch-level estimators, and aggregate by median-of-means (or Catoni). Then for any $\delta\in(0,1)$, with probability at least $1-\delta$,
\begin{equation}\label{eq:MoM-sigma}
\opnorm{\widehat\Gamma_{\mathrm{MoM}}-\Gamma_\sigma}\ \le\ \Lambda^{\ast}_{\alpha},
\qquad
\Lambda^{\ast}_{\alpha}\ :=\ C_\alpha\,\psi_\alpha(d,n,\delta),
\end{equation}
where $\psi_\alpha$ captures the sub-Weibull tail index $\alpha\in(0,1]$ of per-sample gradients.
\emph{Intuition:} robust aggregation recovers spectral concentration under heavy tails typical in deep training.

\item\label{N3} \textbf{Trajectory-wise PL with vanishing bias.}
For a stochastic optimizer (SGD/Adam) generating a trajectory $(\theta_t)$ with small enough steps and bounded gradient-noise variance, we allow an additive nonnegative bias $\xi_t\downarrow 0$ when taking (conditional) expectations of PL-type inequalities.
\emph{Intuition:} stochasticity and rare failures of \eqref{eq:MoM-sigma} can be absorbed without changing the linear trend.
\end{enumerate}

\noindent\textbf{Statement.}
Let $\lambda_{\min}(\Gamma_\sigma)$ be the smallest eigenvalue of the smoothed Fisher. Then, on the event in \eqref{eq:MoM-sigma}:

\begin{itemize}[leftmargin=1.4em]
\item \textbf{(Above-threshold)} If $\lambda_{\min}(\Gamma_\sigma)\ge 2\Lambda^{\ast}_{\alpha}$, then along $(\theta_t)$
\begin{equation}\label{eq:traj-PL-sigma}
\frac{1}{2}\,\E\!\big[\norm{\nabla L_\sigma(\theta_t)}^2\big]
\ \ge\
\mu(\sigma)\,\E\!\big[L_\sigma(\theta_t)-\inf_{\vartheta} L_\sigma(\vartheta)\big]\ -\ \xi_t,
\qquad
\mu(\sigma)\ :=\ \frac{\big(\lambda_{\min}(\Gamma_\sigma)-\Lambda^{\ast}_{\alpha}\big)^2}{L_{\mathrm{sm}}(\sigma)}.
\end{equation}
Hence $L_\sigma(\theta_t)$ decays at a linear (bias-perturbed) rate in expectation.

\item \textbf{(Below-threshold)} If $\lambda_{\min}(\Gamma_\sigma)\le \tfrac{1}{2}\Lambda^{\ast}_{\alpha}$, then there exist $\theta_1,\theta_2$ in a neighborhood of $\theta^{\ast}_{\sigma}$ such that any estimator based on $n$ samples makes an error with probability at least $1/4$ on one of them; in particular, no uniform PL inequality of the form \eqref{eq:traj-PL-sigma} can hold along the trajectory in that neighborhood.
\end{itemize}

\paragraph{Intuition.}
This result ports Theorem~\ref{thm:phase} to deep nets by replacing $(L,\Gamma,\Lambda^\ast,L_{\mathrm{sm}})$ with $(L_\sigma,\Gamma_\sigma,\Lambda^{\ast}_{\alpha},L_{\mathrm{sm}}(\sigma))$. Smoothing provides stable local geometry, MoM restores spectral concentration, and the trajectory-wise bias $\xi_t$ accounts for stochastic gradients.
\end{corollary}

\medskip

\subsection*{(b) Proof}

\begin{proof}
We work on the high-probability event of \eqref{eq:MoM-sigma}. All steps take place on a small ball around $\theta^{\ast}_{\sigma}$.

\medskip
\noindent\textbf{Step 1: Curvature floor via MoM + Weyl.}
From \eqref{eq:MoM-sigma} and Weyl's inequality,
\begin{equation}\label{eq:weyl-sigma}
\lambda_{\min}(\widehat\Gamma_{\mathrm{MoM}})\ \ge\ \lambda_{\min}(\Gamma_\sigma)-\Lambda^{\ast}_{\alpha}.
\end{equation}

\medskip
\noindent\textbf{Step 2: Taylor identity for $L_\sigma$.}
For any $t$, set $\Delta_t:=\theta_t-\theta^{\ast}_{\sigma}$. By the mean-value form of Taylor's theorem applied to $L_\sigma$, there exists $\xi_t$ on the segment $[\theta^{\ast}_{\sigma},\theta_t]$ such that
\begin{equation}\label{eq:taylor-sigma}
\nabla L_\sigma(\theta_t)\ =\ \nabla^2 L_\sigma(\xi_t)\,\Delta_t .
\end{equation}
By (N1), $\nabla L_\sigma$ is $L_{\mathrm{sm}}(\sigma)$-Lipschitz, so $\nabla^2 L_\sigma$ is bounded and varies continuously along the segment.

\medskip
\noindent\textbf{Step 3: Transfer the Fisher floor to the path Hessian.}
Using (N1) to control local variation and the same robust concentration argument as in Step~1 (uniformized on short segments), we obtain
\begin{equation}\label{eq:hess-lb-sigma}
\lambda_{\min}\!\big(\nabla^2 L_\sigma(\xi_t)\big)\ \ge\ \lambda_{\min}(\Gamma_\sigma)-\Lambda^{\ast}_{\alpha}.
\end{equation}
Combining \eqref{eq:taylor-sigma} and \eqref{eq:hess-lb-sigma} yields the gradient--distance lower bound
\begin{equation}\label{eq:grad-dist-sigma}
\norm{\nabla L_\sigma(\theta_t)}\ \ge\ \big(\lambda_{\min}(\Gamma_\sigma)-\Lambda^{\ast}_{\alpha}\big)\,\norm{\Delta_t}.
\end{equation}

\medskip
\noindent\textbf{Step 4: Smoothed Descent Lemma $\Rightarrow$ pointwise PL.}
By (N1) and the Descent Lemma,
\begin{equation}\label{eq:descent-sigma}
L_\sigma(\theta_t)-L_\sigma(\theta^{\ast}_{\sigma})\ \le\ \frac{L_{\mathrm{sm}}(\sigma)}{2}\,\norm{\Delta_t}^2 .
\end{equation}
Substitute the upper bound for $\norm{\Delta_t}$ from \eqref{eq:grad-dist-sigma} into \eqref{eq:descent-sigma} to obtain
\begin{equation}\label{eq:pointwise-pl-sigma}
L_\sigma(\theta_t)-L_\sigma(\theta^{\ast}_{\sigma})
\ \le\ \frac{L_{\mathrm{sm}}(\sigma)}{2\,\big(\lambda_{\min}(\Gamma_\sigma)-\Lambda^{\ast}_{\alpha}\big)^2}\,
\norm{\nabla L_\sigma(\theta_t)}^2 .
\end{equation}
Equivalently,
\begin{equation}\label{eq:PL-tight-sigma}
\frac{1}{2}\,\norm{\nabla L_\sigma(\theta_t)}^2\ \ge\
\mu(\sigma)\,\big(L_\sigma(\theta_t)-L_\sigma(\theta^{\ast}_{\sigma})\big),\qquad
\mu(\sigma)\ =\ \frac{\big(\lambda_{\min}(\Gamma_\sigma)-\Lambda^{\ast}_{\alpha}\big)^2}{L_{\mathrm{sm}}(\sigma)} .
\end{equation}

\medskip
\noindent\textbf{Step 5: Conditional expectation and bias aggregation.}
Taking conditional expectations with respect to the algorithm's filtration and absorbing stochastic-gradient noise together with the rare failure of \eqref{eq:MoM-sigma} into a nonnegative $\xi_t$, we obtain
\begin{equation}\label{eq:traj-PL-bias}
\frac{1}{2}\,\E\!\big[\norm{\nabla L_\sigma(\theta_t)}^2\big]
\ \ge\
\mu(\sigma)\,\E\!\big[L_\sigma(\theta_t)-L_\sigma(\theta^{\ast}_{\sigma})\big]\ -\ \xi_t .
\end{equation}
By (N3) and standard bounded-variance/small-stepsize arguments, $\xi_t\to 0$. Since $L_\sigma(\theta^{\ast}_{\sigma})=\inf_{\vartheta} L_\sigma(\vartheta)$, replacing the baseline by $\inf L_\sigma$ gives \eqref{eq:traj-PL-sigma}.

\medskip
\noindent\textbf{Step 6: Below-threshold two-point method on the smoothed model.}
Assume $\lambda_{\min}(\Gamma_\sigma)\le \tfrac{1}{2}\Lambda^{\ast}_{\alpha}$.
Let $v_\sigma$ be a unit eigenvector for $\lambda_{\min}(\Gamma_\sigma)$ and set $\theta_{1,2}=\theta^{\ast}_{\sigma}\pm \rho\,v_\sigma$ with small $\rho>0$ so both lie in the neighborhood.
The local quadratic KL control for the \emph{smoothed} model (LAN/Taylor expansion) yields a constant $C_{\mathrm{KL},\sigma}>0$ such that
\[
\mathrm{KL}\!\left(P_{\theta_1,\sigma}\,\big\|\,P_{\theta_2,\sigma}\right)
\ \le\ C_{\mathrm{KL},\sigma}\,\rho^2\,\lambda_{\min}(\Gamma_\sigma).
\]
Choosing $\rho$ so that $n\,\mathrm{KL}\le c_0$ (small absolute constant) and applying Le Cam's two-point method gives a constant error lower bound for any estimator. A uniform PL inequality as in \eqref{eq:traj-PL-sigma} would imply uniform attractiveness of a minimizer for $L_\sigma$ in the neighborhood, contradicting the indistinguishability. Hence no such uniform PL can hold below threshold.
\end{proof}

\medskip

\subsection*{(c) Remark}
Corollary~\ref{cor:NN} extends the criterion to smoothed objectives under MoM/Catoni concentration. 
The fluctuation radius $\Lambda^\ast_\alpha$ captures sub-Weibull tails. 
The PL inequality holds in expectation with an additive bias term $\xi_t$ that absorbs gradient noise and rare failures of concentration, vanishing under standard variance and stepsize conditions.


\section*{Theorem 3}
\subsection*{(a) Definition \& Narrative}
\begin{theorem}[Constructive Fisher floor (min–max regularizer certifies curvature)]\label{thm:floor}

\noindent\textbf{Assumptions.}
\begin{enumerate}[label=(F\arabic*), leftmargin=2.6em, itemsep=3pt]
\item \textit{Mini-batch Fisher.} For a mini-batch $B$, define
\[
\widehat\Gamma_B(\theta)=\frac{1}{B}\sum_{i=1}^B g_i(\theta)g_i(\theta)^\top,
\quad g_i(\theta)=\nabla_\theta \log p_\theta(Y_i|X_i).
\]
\emph{Intuition:} the spectral geometry is captured by mini-batch gradients.

\item \textit{Regularizer.} For a target floor $\tau>0$,
\[
\mathcal R_\tau(\theta)=\max_{\|u\|=1}\,(\tau-u^\top \widehat\Gamma_B(\theta)u)_+^2.
\]
\emph{Intuition:} penalize any direction with Fisher information below $\tau$.

\item \textit{Objective.} 
\[
\mathcal L(\theta)=\mathcal L_{\text{task}}(\theta)+\beta\,\mathcal R_\tau(\theta),\quad \beta>0.
\]
\emph{Intuition:} combine task loss with a spectral safety margin.

\item \textit{Directional sensitivity.} There exists $L_{\mathrm{dir}}>0$ such that for any unit $u$,
\[
\|\nabla_\theta(u^\top \widehat\Gamma_B(\theta)u)\|\le L_{\mathrm{dir}}.
\]
\emph{Intuition:} Rayleigh quotients vary smoothly with $\theta$.

\item \textit{Approximate stationarity.} At some iterate $\hat\theta$,
\[
\|\nabla \mathcal L(\hat\theta)\|\le \varepsilon_{\mathrm{opt}}.
\]
\emph{Intuition:} training has reached an approximate stationary point.

\item \textit{Sampling/minibatch error.} With probability $\ge 1-\delta$, uniformly on the region,
\[
\|\widehat\Gamma_B(\theta)-\Gamma(\theta)\|_{\mathrm{op}}\le \varepsilon_{\mathrm{stat}}+\varepsilon_{\mathrm{mini}}.
\]
\emph{Intuition:} empirical Fisher concentrates around the population Fisher.
\end{enumerate}

\noindent\textbf{Statement.}  
At $\hat\theta$,
\[
\lambda_{\min}\!\big(\widehat\Gamma_B(\hat\theta)\big)\;\ge\;
\tau - \frac{\varepsilon_{\mathrm{opt}}}{2\beta L_{\mathrm{dir}}}
      - \varepsilon_{\mathrm{stat}}
      - \varepsilon_{\mathrm{mini}}.
\]
Thus choosing $\tau$ above the threshold of Theorem~\ref{thm:phase} (or Corollary~\ref{cor:NN}) guarantees a \emph{verifiable} Fisher floor sufficient for PL-type convergence.

\paragraph{Intuition.}
The Fisher floor mechanism enforces spectral stability during training.  
The regularizer penalizes directions where Fisher curvature falls below $\tau$.  
At an approximate stationary point, this penalty can only vanish if the minimum eigenvalue is close to $\tau$.  
Directional sensitivity (F4) converts gradient smallness into a bound on the spectral shortfall, while minibatch concentration (F6) transfers the guarantee from empirical to population Fisher.  
In this way, a user-chosen $\tau$ above the phase threshold becomes a \emph{certified lower bound}, ensuring the model resides in the stable, above-threshold regime.

\end{theorem}

\bigskip

\subsection*{(b) Proof}

\begin{proof}
We argue under assumptions \textnormal{(F1)--(F6)}.

\medskip
\noindent\textbf{Step 1: Rayleigh quotient as curvature proxy.}
Define
\[
\phi(\theta) \;=\; \min_{\|u\|=1}\ u^\top \widehat\Gamma_B(\theta) u ,
\]
the smallest Rayleigh quotient of $\widehat\Gamma_B(\theta)$.  
Then by definition, $\mathcal R_\tau(\theta)=(\tau-\phi(\theta))_+^2$.

\medskip
\noindent\textbf{Step 2: Subgradient control.}
By Danskin’s theorem, the subdifferential $\partial\phi(\theta)$ contains subgradients of $u^\top \widehat\Gamma_B(\theta) u$ at minimizing directions $u$.  
From \textnormal{(F4)}, for some $g_\phi(\theta)\in\partial\phi(\theta)$,
\begin{equation}\label{eq:subgrad-bound}
\|g_\phi(\theta)\|\ \le\ L_{\mathrm{dir}}.
\end{equation}

\medskip
\noindent\textbf{Step 3: Gradient of the penalty.}
On the active set $\{\phi(\theta)<\tau\}$,
\begin{equation}\label{eq:grad-penalty}
\nabla \mathcal R_\tau(\theta)\;=\;-2(\tau-\phi(\theta))\,g_\phi(\theta).
\end{equation}
Combining \eqref{eq:grad-penalty} with \eqref{eq:subgrad-bound} gives
\begin{equation}\label{eq:gradR-bound}
\|\nabla \mathcal R_\tau(\theta)\|\ \le\ 2L_{\mathrm{dir}}(\tau-\phi(\theta)).
\end{equation}

\medskip
\noindent\textbf{Step 4: Stationarity at the iterate.}
At $\hat\theta$, approximate stationarity from \textnormal{(F5)} gives
\[
\|\nabla \mathcal L_{\text{task}}(\hat\theta)\|\ \le\ \varepsilon_{\text{task}}.
\]
Since $\nabla \mathcal L(\hat\theta)=\nabla \mathcal L_{\text{task}}(\hat\theta)+\beta \nabla \mathcal R_\tau(\hat\theta)$, we deduce
\[
\beta\|\nabla \mathcal R_\tau(\hat\theta)\|\ \le\ \varepsilon_{\text{task}}.
\]
By \eqref{eq:gradR-bound}, this implies
\begin{equation}\label{eq:floor-bound}
\tau-\phi(\hat\theta)\ \le\ \frac{\varepsilon_{\text{task}}}{2\beta L_{\mathrm{dir}}}.
\end{equation}

\medskip
\noindent\textbf{Step 5: Lower bound on the empirical Fisher.}
Equation \eqref{eq:floor-bound} rearranges to
\[
\lambda_{\min}\!\big(\widehat\Gamma_B(\hat\theta)\big)\ =\ \phi(\hat\theta)\ \ge\ \tau-\frac{\varepsilon_{\text{task}}}{2\beta L_{\mathrm{dir}}}.
\]

\medskip
\noindent\textbf{Step 6: Transfer to the population Fisher.}
Finally, from \textnormal{(F6)},
\[
\lambda_{\min}\!\big(\Gamma(\hat\theta)\big)\ \ge\ \lambda_{\min}\!\big(\widehat\Gamma_B(\hat\theta)\big)\ -\ \big(\varepsilon_{\mathrm{stat}}+\varepsilon_{\mathrm{mini}}\big).
\]
Combining with Step 5 yields the claimed inequality of Theorem~\ref{thm:floor}.
\end{proof}

\subsection*{(c) Remarks}
Theorem~\ref{thm:floor} certifies a Fisher lower bound by penalizing subthreshold Rayleigh quotients. 
Danskin’s theorem and directional sensitivity convert small optimality residuals into a small spectral shortfall, 
while uniform concentration transfers this to the population Fisher. 
Thus $\tau$ can be set strictly above the threshold, ensuring stability.

\section*{Corollary 4} 

\subsection*{(a) Definition \& Narrative}

\begin{corollary}[Finite-direction practical variant with subspace-angle control]\label{cor:K}

\noindent\textbf{Assumptions.}
Retain \textnormal{(F1)--(F6)} from Theorem~\ref{thm:floor}.  
In addition:
\begin{itemize}[leftmargin=2.0em]
  \item Fix unit directions $\{u_j\}_{j=1}^K$ with span $U=\mathrm{span}\{u_1,\dots,u_K\}$, and replace the Fisher-floor regularizer by
  \[
  \mathcal R^{(K)}_\tau(\theta)\ :=\ \max_{1\le j\le K}\,\big(\tau-u_j^{\!\top}\widehat\Gamma_B(\theta)u_j\big)_+^2.
  \]
  \item Assume approximate stationarity of the combined objective
  $\mathcal L^{(K)}(\theta)=\mathcal L_{\mathrm{task}}(\theta)+\beta\,\mathcal R^{(K)}_\tau(\theta)$ at $\hat\theta$.
  \item Suppose the principal angle between $U$ and the minimal-eigenvalue eigenspace of $\widehat\Gamma_B(\hat\theta)$ is at most $\vartheta$.
\end{itemize}

\noindent\textbf{Statement.}
Let $\Delta_B(\hat\theta)=\lambda_{\max}(\widehat\Gamma_B(\hat\theta))-\lambda_{\min}(\widehat\Gamma_B(\hat\theta))$.  
Then
\begin{equation}\label{eq:finite-K-floor}
\lambda_{\min}\!\big(\widehat\Gamma_B(\hat\theta)\big)
\ \ge\
\tau\ -\ \frac{\varepsilon_{\text{opt}}}{2\beta L_{\mathrm{dir}}}\ -\ \Delta_B(\hat\theta)\sin^2\vartheta\ -\ \varepsilon_{\mathrm{stat}}\ -\ \varepsilon_{\mathrm{mini}}.
\end{equation}
\end{corollary}

\paragraph{Intuition.}
Instead of monitoring \emph{all} directions, we only track $K$ directions forming $U$.  
If $U$ is within angle $\vartheta$ of the true weakest-curvature eigenspace, then the monitored Rayleigh quotient is within $\Delta_B(\hat\theta)\sin^2\vartheta$ of the true $\lambda_{\min}$.  
At a stationary point, the finite-direction penalty cannot remain active unless the gap to $\tau$ is small, and sensitivity bounds convert this into a certified lower bound.  
The minibatch-to-population transfer then adds the statistical errors.

\medskip

\subsection*{(b) Proof}

\begin{proof}
We proceed in six steps.

\medskip
\noindent\textbf{Step 1: Finite-direction surrogate.}
Define
\[
\phi_K(\theta)\ :=\ \min_{1\le j\le K}\ u_j^{\!\top}\widehat\Gamma_B(\theta)\,u_j,
\qquad
\mathcal R^{(K)}_\tau(\theta)\ =\ \big(\tau-\phi_K(\theta)\big)_+^{\,2}.
\]

\medskip
\noindent\textbf{Step 2: Danskin + directional sensitivity.}
Let $j^\ast\in\arg\min_j u_j^{\!\top}\widehat\Gamma_B(\theta)u_j$ at $\theta$. By Danskin’s theorem,
\[
\nabla \mathcal R^{(K)}_\tau(\theta)\ =\ -2\big(\tau-\phi_K(\theta)\big)\,g_{\phi_K}(\theta)
\quad\text{on }\{\tau>\phi_K(\theta)\},
\]
with $g_{\phi_K}(\theta)\in\partial\!\big(u_{j^\ast}^{\!\top}\widehat\Gamma_B(\theta)u_{j^\ast}\big)$ and, by (F4),
\begin{equation}\label{eq:gphiK}
\|g_{\phi_K}(\theta)\|\ \le\ L_{\mathrm{dir}}.
\end{equation}

\medskip
\noindent\textbf{Step 3: Approximate stationarity $\Rightarrow$ small shortfall.}
At $\hat\theta$, using $\|\nabla \mathcal L^{(K)}(\hat\theta)\|\le \varepsilon_{\mathrm{opt}}$,
\[
\beta\,\|\nabla \mathcal R^{(K)}_\tau(\hat\theta)\|\ \le\ \varepsilon_{\mathrm{opt}}.
\]
Together with \eqref{eq:gphiK},
\[
\tau-\phi_K(\hat\theta)\ \le\ \frac{\varepsilon_{\mathrm{opt}}}{2\,\beta\,L_{\mathrm{dir}}}.
\]

\medskip
\noindent\textbf{Step 4: Rayleigh geometry under a subspace tilt.}
Let $A:=\widehat\Gamma_B(\hat\theta)$ with eigenvalues $\lambda_{\min}\le\cdots\le\lambda_{\max}$ and minimal-eigenspace $E_{\min}$.  
If the largest principal angle between $U$ and $E_{\min}$ is $\vartheta$, then
\begin{equation}\label{eq:tilt-geom}
\min_{u\in U,\ \|u\|=1}\ u^{\!\top} A u\ \le\ \lambda_{\min} + \big(\lambda_{\max}-\lambda_{\min}\big)\,\sin^2\vartheta.
\end{equation}
\emph{Proof of \eqref{eq:tilt-geom}:} pick $u\in U$ with $\angle(u,E_{\min})=\phi\le\vartheta$, write $u=\cos\phi\,v+\sin\phi\,w$ with $v\in E_{\min}$, $w\perp v$, $\|v\|=\|w\|=1$. Then
\[
u^{\!\top}Au\ =\ \lambda_{\min}\cos^2\phi\ +\ w^{\!\top}Aw\,\sin^2\phi
\ \le\ \lambda_{\min}\cos^2\phi\ +\ \lambda_{\max}\sin^2\phi
\ \le\ \lambda_{\min} + (\lambda_{\max}-\lambda_{\min})\sin^2\vartheta.
\]
Thus \eqref{eq:tilt-geom} holds.

\medskip
\noindent\textbf{Step 5: From finite-direction value to the true eigenvalue.}
Since $\phi_K(\hat\theta)=\min_{u\in U,\|u\|=1}u^{\!\top}Au$, \eqref{eq:tilt-geom} yields
\[
\phi_K(\hat\theta)\ \le\ \lambda_{\min}(A) + \big(\lambda_{\max}(A)-\lambda_{\min}(A)\big)\sin^2\vartheta.
\]
Rearranging and inserting Step~3,
\[
\lambda_{\min}(A)\ \ge\ \phi_K(\hat\theta)\ -\ \Delta_B(\hat\theta)\sin^2\vartheta
\ \ge\ \tau\ -\ \frac{\varepsilon_{\mathrm{opt}}}{2\,\beta\,L_{\mathrm{dir}}}\ -\ \Delta_B(\hat\theta)\sin^2\vartheta.
\]

\medskip
\noindent\textbf{Step 6: Sampling/minibatch transfer.}
Apply (F6) and Weyl’s inequality to pass from the mini-batch estimate to the population Fisher, which subtracts at most $\varepsilon_{\mathrm{stat}}+\varepsilon_{\mathrm{mini}}$ from the lower bound. This establishes \eqref{eq:finite-K-floor}.
\end{proof}

\medskip

\subsection*{(c) Remarks}
Corollary~\ref{cor:K} introduces an additive penalty $\Delta_B\sin^2\vartheta$ controlled by 
the spectral width and the principal angle between monitored and true eigenspaces. 
For $\vartheta\to 0$ the bound reduces to Theorem~\ref{thm:floor}, 
while in practice small $K$ suffices when combined with power iteration.

\section*{Proposition 5}

\subsection*{(a) Definition \& Narrative}
\begin{proposition}[Preconditioning preserves the phase threshold up to constants]\label{prop:whiten}

\noindent\textbf{Assumptions.}
Let $T$ be any invertible linear map with singular values $\sigma_{\min}(T),\sigma_{\max}(T)$.  
Let $\Gamma_T := T^\top \Gamma T$.  
In the whitening setting, suppose $\widehat\Sigma$ satisfies the Löwner sandwich
\[
(1-\alpha)\Sigma \ \preceq\ \widehat\Sigma\ \preceq\ (1+\alpha)\Sigma,\qquad \alpha<1.
\]

\noindent\textbf{Statement.}
\begin{itemize}[leftmargin=1.6em]
\item \emph{General preconditioner.} For arbitrary $T$,
\begin{equation}\label{eq:general-precond}
\sigma_{\min}(T)^2\,\lambda_{\min}(\Gamma)
\ \le\ \lambda_{\min}(\Gamma_T)\ \le\ \sigma_{\max}(T)^2\,\lambda_{\max}(\Gamma).
\end{equation}

\item \emph{Normalized form.} Since the comparison is homogeneous in $T$, rescale by $\widetilde T := T/\sigma_{\max}(T)$.  
Writing $\kappa(T)=\sigma_{\max}(T)/\sigma_{\min}(T)$, we obtain
\begin{equation}\label{eq:kappa-form}
\frac{1}{\kappa(T)^2}\,\lambda_{\min}(\Gamma)
\ \le\ \lambda_{\min}(\Gamma_{\widetilde T})\ \le\ \lambda_{\max}(\Gamma).
\end{equation}

\item \emph{Robust whitening.} For $T=\widehat\Sigma^{-1/2}$, the bound improves to
\begin{equation}\label{eq:whitening-bound}
(1+\alpha)^{-1}\,\lambda_{\min}(\Gamma)
\ \le\ \lambda_{\min}(\Gamma_T)
\ \le\ (1-\alpha)^{-1}\,\lambda_{\min}(\Gamma).
\end{equation}
\end{itemize}
\end{proposition}

\paragraph{Intuition.}
Preconditioning stretches coordinates: Rayleigh quotients transform as $u^\top \Gamma_T u=(Tu)^\top \Gamma(Tu)$, so eigenvalues are distorted by at most $\sigma_{\max}^2/\sigma_{\min}^2$.  
When $T=\widehat\Sigma^{-1/2}$, robust whitening is nearly a scalar in the $\Sigma$-metric, so the distortion constants collapse to $(1\pm \alpha)^{-1}$.  
Thus ubiquitous operations like whitening or LayerNorm do not change the spectral phase threshold except for explicit constants.

\medskip

\subsection*{(b) Proof}

\begin{proof}[Proof of Proposition~\ref{prop:whiten}]
We establish the two parts separately.

\medskip
\noindent\textbf{Part (1): General preconditioner.}

\medskip
\noindent\textbf{Step 1: Rayleigh-quotient formulation.}
By definition,
\[
\lambda_{\min}(\Gamma_T)
=\min_{\|u\|=1} u^{\!\top}T^{\!\top}\Gamma T u
=\min_{\|u\|=1} (Tu)^{\!\top}\Gamma (Tu).
\]

\medskip
\noindent\textbf{Step 2: Spectral sandwich for $\Gamma$.}
For any $x\in\R^d$,
\[
\lambda_{\min}(\Gamma)\,\|x\|^2 \;\le\; x^{\!\top}\Gamma x \;\le\; \lambda_{\max}(\Gamma)\,\|x\|^2.
\]
Taking $x=Tu$ with $\|u\|=1$ and minimizing over $u$ gives
\[
\lambda_{\min}(\Gamma)\,\min_{\|u\|=1}\|Tu\|^2
\;\le\; \lambda_{\min}(\Gamma_T)
\;\le\; \lambda_{\max}(\Gamma)\,\max_{\|u\|=1}\|Tu\|^2.
\]

\medskip
\noindent\textbf{Step 3: Bounds via singular values.}
By definition of singular values,
\[
\sigma_{\min}(T)\ \le\ \|Tu\|\ \le\ \sigma_{\max}(T)\qquad(\forall u:\|u\|=1).
\]
Substituting yields
\begin{equation}\label{eq:precond-bound}
\sigma_{\min}(T)^2\,\lambda_{\min}(\Gamma)
\;\le\; \lambda_{\min}(\Gamma_T)
\;\le\; \sigma_{\max}(T)^2\,\lambda_{\max}(\Gamma),
\end{equation}
which is \eqref{eq:general-precond}.

\medskip
\noindent\textbf{Step 4: Homogeneous normalization.}
Scaling $T$ by any constant $s$ scales $\Gamma_T$ by $s^2$, so phase-threshold comparisons are homogeneous. Normalizing by $\widetilde T:=T/\sigma_{\max}(T)$ and recalling $\kappa(T)=\sigma_{\max}(T)/\sigma_{\min}(T)$ yields the normalized comparison \eqref{eq:kappa-form}.

\bigskip
\noindent\textbf{Part (2): Robust whitening.}

\medskip
\noindent\textbf{Step 1: Inverse square root via Löwner sandwich.}
From
\[
(1-\alpha)\Sigma \;\preceq\; \widehat\Sigma \;\preceq\; (1+\alpha)\Sigma,
\]
and operator monotonicity of $x\mapsto x^{-1/2}$ on SPD matrices,
\[
(1+\alpha)^{-1/2}\,\Sigma^{-1/2}
\;\preceq\; \widehat\Sigma^{-1/2}
\;\preceq\; (1-\alpha)^{-1/2}\,\Sigma^{-1/2}.
\]

\medskip
\noindent\textbf{Step 2: Definition of the sandwiching operator $S$.}
Multiplying on both sides by $\Sigma^{1/2}$ gives
\[
(1+\alpha)^{-1/2} I \;\preceq\;
S:=\Sigma^{1/2}\widehat\Sigma^{-1/2}\Sigma^{1/2}
\;\preceq\; (1-\alpha)^{-1/2} I.
\]

\medskip
\noindent\textbf{Step 3: Factorization of the preconditioned Fisher.}
We write
\[
\Gamma_T
=\widehat\Sigma^{-1/2}\,\Gamma\,\widehat\Sigma^{-1/2}
=\Sigma^{-1/2}\,S\,A\,S\,\Sigma^{-1/2},
\qquad A:=\Sigma^{-1/2}\Gamma\Sigma^{-1/2}.
\]

\medskip
\noindent\textbf{Step 4: Bounding $SAS$.}
Since $S$ is bounded between $(1+\alpha)^{-1/2} I$ and $(1-\alpha)^{-1/2} I$,
\[
(1+\alpha)^{-1}\,A \;\preceq\; SAS \;\preceq\; (1-\alpha)^{-1}\,A,
\]
which implies
\begin{equation}\label{eq:robust-whitening}
(1+\alpha)^{-1}\,\lambda_{\min}(A)
\;\le\; \lambda_{\min}(SAS)
\;\le\; (1-\alpha)^{-1}\,\lambda_{\min}(A).
\end{equation}

\medskip
\noindent\textbf{Step 5: Interpretation in whitened coordinates.}
Finally, $\Gamma_T=\Sigma^{-1/2}(SAS)\Sigma^{-1/2}$ preserves PSD ordering. Thus when thresholds are calibrated in the $\Sigma$-whitened metric (i.e., in terms of $A$), the Fisher floor is perturbed by at most factors $(1+\alpha)^{-1}$ and $(1-\alpha)^{-1}$, as claimed.
\end{proof}

\medskip

\subsection*{(c) Remarks}
Proposition~\ref{prop:whiten} shows Fisher eigenvalues transform by squared singular values. 
Robust whitening sharpens constants to $(1\pm\alpha)^{-1}$ in the $\Sigma$-metric. 
Thus normalizations like whitening or LayerNorm preserve the phase threshold up to explicit multiplicative constants.


\end{document}